\documentclass{article}

\usepackage{times}
\usepackage{graphicx} 
\usepackage{subfigure} 

\usepackage{natbib}

\usepackage{mystyle}
\usepackage{symbol}
\usepackage{mlsymbols}

\usepackage[accepted]{icml2016}

\icmltitlerunning{Expressiveness of Rectifier Networks}

\newcommand{\relu}{ReLU}
\newcommand{\relus}{ReLUs}
\newcommand{\threshold}[1]{\sgn\p{#1}}
\newcommand{\rectifier}[1]{R\p{#1}}

\begin{document}

\twocolumn[
\icmltitle{Expressiveness of Rectifier Networks}

\icmlauthor{Xingyuan Pan}{xpan@cs.utah.edu}
\icmlauthor{Vivek Srikumar}{svivek@cs.utah.edu}
\icmladdress{The University of Utah,
            Salt Lake City, UT 84112, USA}

\icmlkeywords{machine learning, artificial neural network, rectified linear unit, expressiveness}

\vskip 0.3in
]

\begin{abstract}
  Rectified Linear Units (ReLUs) have been shown to ameliorate the
  vanishing gradient problem, allow for efficient backpropagation, and
  empirically promote sparsity in the learned parameters. They have
  led to state-of-the-art results in a variety of applications.
  However, unlike threshold and sigmoid networks, ReLU networks are
  less explored from the perspective of their expressiveness. This
  paper studies the expressiveness of ReLU networks. We characterize
  the decision boundary of two-layer ReLU networks by constructing
  functionally equivalent threshold networks. We show that while the
  decision boundary of a two-layer ReLU network can be captured by a
  threshold network, the latter may require an exponentially larger
  number of hidden units. We also formulate sufficient conditions for
  a corresponding logarithmic reduction in the number of hidden units
  to represent a sign network as a ReLU network. Finally, we
  experimentally compare threshold networks and their much smaller
  ReLU counterparts with respect to their ability to learn from
  synthetically generated data.
\end{abstract}


\section{Introduction}
\label{sec:intro}

A neural network is characterized by its architecture, the choices of
activation functions, and its parameters. We see several activation
functions in the literature -- the most common ones being threshold,
logistic, hyperbolic tangent and rectified linear units (\relus). In
recent years, deep neural networks with rectifying neurons -- defined
as $\rectifier{x} = \max(0, x)$ -- have shown state-of-the-art
performance in several tasks such as image and speech classification
\cite{Glorot2011,nair2010rectified,Krizhevsky2012,maas2013rectifier,zeiler2013rectified}.

\relus~ possess several attractive computational properties. First,
compared to deep networks with sigmoidal activation units, \relu{}
networks are less affected by the vanishing gradient
problem~\cite{bengio1994learning,hochreiter1998vanishing,Glorot2011}.
Second, rectifying neurons encourage sparsity in the hidden
layers~\cite{Glorot2011}. Third, gradient back propagation is
efficient because of the piece-wise linear nature of the function. For
example, \citet{Krizhevsky2012} report that a convolutional neural
network with \relus{} is six times faster than an equivalent one with
hyperbolic tangent neurons. Finally, they have been empirically been
shown to generalize very well.

Despite these clear computational and empirical advantages, the
expressiveness of rectifier units is less studied unlike sigmoid and
threshold units.
In this paper, we address the following question: {\em Which Boolean
  functions do \relu{} networks express?} We analyze the
expressiveness of shallow \relu{} networks by characterizing their
equivalent threshold networks. The goal of our analysis is to offer a
formal understanding for the successes of \relus{} by comparing them
to threshold functions that are well
studied~\citep[e.g.][]{hajnal1993threshold}. To this end, the
contributions of this paper are:

\begin{enumerate}
\item We provide a constructive proof that two layer \relu{} networks
  are equivalent to exponentially larger threshold networks.
  Furthermore, we show that there exist two layer \relu{} networks
  that {\em cannot} be represented by any smaller threshold networks.
\item We use this characterization to define a sufficient condition
  for compressing an arbitrary threshold network into a
  logarithmically smaller \relu{} network.
\item We identify a relaxation of this condition that is applicable if
  we treat hidden layer predictions as a multiclass classification,
  thus requiring equivalence of hidden layer states instead of the
  output state.
\end{enumerate}


\subsection{Expressiveness of Networks: Related Work}
\label{sec:previous}

From the learning point of view, the choice of an activation function
is driven by two related aspects: the expressiveness of a given
network using the activation function, and the computational
complexity of learning. Though this work studies the former, we
briefly summarize prior work along both these lines.

Any continuous function can be approximated to arbitrary accuracy with
only one hidden layer of sigmoid units~\cite{Cybenko1989}, leading to
neural networks being called ``universal approximators''. With two
layers, even discontinuous functions can be represented. Moreover, the
approximation error (for real-valued outputs) is insensitive to the
choice of activation functions from among several commonly used ones
\cite{DasGupta1993}, provided we allow the size of the network to
increase polynomially and the number of layers to increase by a
constant factor. Similarly, two layer threshold networks are capable
of representing any Boolean function. However, these are existence
statements; for a general target function, the number of hidden units
may be exponential in the input dimensionality.
\citet{maass1991computational,maass1994comparison} compare sigmoid
networks with threshold networks and point out that the former can be
more expressive than similar-sized threshold networks.



There has been some recent work that looks at the expressiveness of
feed-forward \relu{} networks. Because the rectifier function is
piece-wise linear, any network using only \relus{} can only represent
piece-wise linear functions. Thus, the number of linear partitions of
input space by the network can be viewed as a measure of its
expressiveness. \citet{Pascanu2014} and \citet{Montufar2014} show that
for the same number of \relus{}, a deep architecture can represent
functions with exponentially more linear regions than a shallow
architecture. While more linear regions indicate that more complex
functions can be represented, it does not directly tell us how
expressive a function is; at prediction time, we cannot directly
correlate the number of regions to the way we make the prediction.
Another way of measuring the expressiveness of a feed-forward networks
is by considering its classification error; \citet{Telgarsky2015}
compares shallow and deep \relu{} networks in this manner.

The learning complexity of neural networks using various activation
functions has also been studied. For inputs from the Boolean
hypercube, the two-layer networks with threshold activation functions
is not efficiently
learnable~\citep[e.g.][]{blum1992training,Klivans2006,Daniely2014}.
Without restricting the weights, two layer networks with sigmoid or
\relu{} activations are also not efficiently learnable. We also refer
the reader to \citet{livni2014computational} that summarizes and
describes positive and negative learnability results for various
activations.


\section{What do \relus~express?}
\label{sec:expressiveness}
To simplify our analysis, we primarily focus on shallow networks with
one hidden layer with $n$ units and a single binary output. In all
cases, the hidden layer neurons are the object of study. The output
activation function is always the threshold function.
In the rest of the paper, we use boldfaced letters to denote vectors.
Input feature vectors and output binary labels are represented by
$\bx$ and $y \in \{\pm 1\}$ respectively. The number of hidden units
is $n$. The weights and bias for the $k^{th}$ rectifier are $\bu_k$
and $b_k$; the weights and bias for the $k^{th}$ sign units are
$\bv_k$ and $d_k$. The weights for the output unit are $w_1$ through
$w_n$, and its the bias is $w_0$.

\subsection{Threshold networks}
\label{sec:threshold-expressiveness}

Before coming to the main results, we will first review the
expressiveness of threshold networks. Assuming there are $n$ hidden
units and one output unit, the output of the network can be written as
\begin{equation}
  y = \threshold{w_0 + \sum_{k=1}^n w_k \threshold{\bv_k \cdot \bx + d_k}}.
\end{equation}
Here, both hidden and output activations are the sign function, which
is defined as $\threshold{x}=1 \text{ if } x \ge 0$, and $-1$
otherwise. Each hidden unit represents one hyperplane (parameterized
by $\bv_k$ and $d_k$) that bisects the input space into two half
spaces. By choosing different weights in the hidden layer we can
obtain arbitrary arrangement of $n$ hyperplanes. The theory of
hyperplane arrangement~\cite{Zaslavsky1975} tells us that for a
general arrangement of $n$ hyperplanes in $d$ dimensions, the space is
divided into $\sum_{s=0}^{d} \binom{n}{s}$ regions. The output unit
computes a linear combination of the hidden output (using the $w$'s)
and thresholds it. Thus, for various values of the $w$'s, threshold
networks can express intersections and unions of those regions.
Figure~\ref{fig:sign_boundary} shows an example of the decision
boundary of a two-layer network with three threshold units in the
hidden layer.
\begin{figure}[htbp] 
  \centering
  \begin{tikzpicture}[scale=0.65]
  \draw[step=1cm,gray,very thin] (-3,-3) grid (3,3);

  \filldraw[fill=red!40, draw=none,fill opacity=0.5](-2,-3)--(-0.8571,-1.2857)--(0.3,-0.9)--(1,-3)--(-2,-3);
  \filldraw[pattern=crosshatch, pattern color=green, draw=none](-2,-3)--(-0.8571,-1.2857)--(0.3,-0.9)--(1,-3)--(3,-3)--(3,3)--(-3,3)--(-3,-3)--(-2,-3);

  \draw (-2,-3) -- (2, 3);
  \draw (-3,-2) -- (3,0);
  \draw (-1,3) -- (1,-3);

  \coordinate[label=left:$\bv_1$] (w1) at ($ (1, 1.5) ! 0.5cm ! 90:(2,3) $);
  \coordinate[label=left:$\bv_2$] (w2) at ($ (-2, -1.667) ! 0.5cm ! 90:(3,1) $);
  \coordinate[label=right:$\bv_3$] (w3) at ($ (0.667, -2) ! 0.5cm ! -90:(-1,3) $);

  \draw[arrows=->, very thick] (1,1.5) -- (w1); 
  \draw[arrows=->, very thick] (-2,-1.667) -- (w2); 
  \draw[arrows=->, very thick] (0.667,-2) -- (w3); 

\end{tikzpicture}

  \caption{An example of the decision boundary of a two-layer network
    in two dimensions, with three threshold units in the hidden layer.
    The arrows point towards the half-space that is classified as
    positive (the green checked region).}
  \label{fig:sign_boundary}

\end{figure}
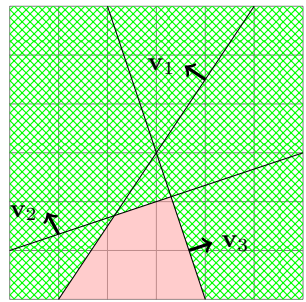

\subsection{Rectifier networks}
\label{subsec:relu-decisions}
In this section, we will show that the decision boundary of every
two-layer neural network with rectifier activations can be represented
using a network with threshold activations with two or three layers.
However, the number of hidden threshold units can be exponential
compared to the number of hidden rectifier units.

Consider a network with one hidden layer of $n$ \relus{}, denoted by
$\rectifier{\cdot}$. For a $d$ dimensional input $\bx$, the output $y$
is computed as
\begin{equation}\label{eq:general}
  y = \threshold{w_0 + \sum_{k=1}^n w_k \rectifier{\bu_k \cdot \bx + b_k}}.
\end{equation}
Here $\bu_k$ and $b_k$ are weight and bias parameters for the
\relus{} in the hidden layer, and the $w_k$'s parameterize the output
unit. To simplify notation, we will use $a_k$ to denote the the
pre-activation input of the $k^{th}$ hidden unit. That is,
$a_k(\bx) = \bu_k \cdot \bx + b_k$. This allows us to simplify the
output as
$\threshold{w_0 + \sum_{k \in [n]} w_k \rectifier{a_k(\bx)}}$. Here,
$[n]$ is the set of positive integers not more than $n$. Note that
even when not explicitly mentioned, each $a_k$ depends on the $\bu_k$
and the $b_k$ parameters.

By definition of the rectifier, for any real number $c$, we have
$c\rectifier{x} = \sgn(c)\rectifier{|c|x}$. Thus, we can absorb
$\vert w_k \vert$ into the rectifier function in
Eq.~\eqref{eq:general} without losing generality. That is, other than
$w_0$, all the other output layer weights are only relevant up to sign
because their magnitude can be absorbed into hidden layer weights. We
can partition the hidden units into two sets $\sP$ and $\sN$,
depending on the sign of the corresponding $w_k$. That is, let
$\sP = \{k: k \in [n] \text{ and } w_k = +1\}$ and let
$\sN = \{k: k \in [n] \text{ and } w_k = -1\}$. We will refer to these
partitions as the {\em positive} and {\em negative} hidden units
respectively.

This observation lets us state the general form of two-layer \relu{}
networks as:
\begin{equation}\label{eq:general_relu}
  y = \threshold{w_0 + \sum_{k \in \sP} \rectifier{a_k(\bx)} - \sum_{k \in \sN} \rectifier{a_k(\bx)}}.
\end{equation}
The following two layer rectifier network will serve as our
running example through the paper:
\begin{equation}\label{eq:example_relu}
  y = \threshold{w_0 + \rectifier{a_1(\bx)} - \rectifier{a_2(\bx)} - \rectifier{a_3(\bx)}}.
\end{equation}
This network consists of three \relus{} in the hidden layer, one of
which positively affects the pre-activation output and the other two
decrease it. Hence, the set $\sP=\{1\}$ and the set $\sN=\{2, 3\}$. 

Using the general representation of a two layer network with rectifier
hidden units (Eq. \eqref{eq:general_relu}), we can now state our main
theorem that analyzes the decision boundary of rectifier networks.
%

\begin{theorem}[Main Theorem]\label{theorem:main}
  Consider a two-layer rectifier network with $n$ hidden units
  represented in its general form (Eq.~\eqref{eq:general_relu}). Then,
  for any input $\bx$, the following conditions are equivalent:
  \begin{enumerate}
  \item The network classifies the example $\bx$ as
    positive. \label{theorem:main:condition1}
  \item There exists a subset $\sS_1$ of $\sP$ such that, for every
    subset $\sS_2$ of $\sN$, we have
    $w_0 + \sum_{k \in \sS_1} a_k(\bx) - \sum_{k \in \sS_2} a_k(\bx)
    \ge 0$. \label{theorem:main:condition2}
  \item For every subset $\sS_2$ of $\sN$, there exists a subset
    $\sS_1$ of $\sP$ such that
    $w_0 + \sum_{k \in \sS_1} a_k(\bx) - \sum_{k \in \sS_2} a_k(\bx)
    \ge 0$. \label{theorem:main:condition3}
  \end{enumerate}
\end{theorem}

Before discussing the implications of the theorem, let us see how it
applies to our running example in Eq.~\eqref{eq:example_relu}. In this
example, $\sP$ has two subsets: $\emptyset$ and $\{1\}$, and $\sN$ has
four subsets: $\emptyset$, $\{2\}$, $\{3\}$ and $\{2, 3\}$. The first
and second conditions of Theorem~\ref{theorem:main} indicate that the
prediction is positive if, and only if, at least one of the sets of
conditions in Figure~\ref{fig:conditions-for-positive} hold in
entirety.
\begin{figure*}[t]
  \centering
  \footnotesize
  $\begin{cases}
      w_0 \ge 0,                                 & (\mbox{with }\sS_1=\emptyset, \sS_2=\emptyset) \\
      w_0 - a_2(\bx) \ge 0,                      & (\mbox{with }\sS_1=\emptyset, \sS_2=\{2\})     \\
      w_0 - a_3(\bx) \ge 0,                      & (\mbox{with }\sS_1=\emptyset, \sS_2=\{3\})     \\
      w_0 - a_2(\bx) - a_3(\bx) \ge 0.           & (\mbox{with }\sS_1=\emptyset, \sS_2=\{2, 3\})
    \end{cases}$~(or)~$
    \begin{cases}
      w_0 + a_1(\bx)\ge 0,                       & (\mbox{with }\sS_1=\{1\}, \sS_2=\emptyset)     \\
      w_0+ a_1(\bx) - a_2(\bx) \ge 0,            & (\mbox{with } \sS_1=\{1\}, \sS_2=\{2\})        \\
      w_0+ a_1(\bx) - a_3(\bx) \ge 0,            & (\mbox{with } \sS_1=\{1\},\sS_2=\{3\})         \\
      w_0+ a_1(\bx) - a_2(\bx) - a_3(\bx) \ge 0. & (\mbox{with } \sS_1=\{1\},\sS_2=\{2, 3\})
    \end{cases}$                   
  \caption{The sets of inequalities that should hold for the running
    example to predict an input as positive.}
  \label{fig:conditions-for-positive}
\end{figure*}

Each big left brace indicates a system of inequalities all of which
should hold; thus essentially the conjunction of the individual
inequalities contained within it. We can interpret of the subsets of
$\sP$ as certificates. In order for the output of
Eq.~\eqref{eq:example_relu} to be positive, we need at least one
certificate $\sS_1$ (one subset of $\sP$) such that for every subset
$\sS_2$ of $\sN$,
$w_0 + \sum_{k \in \sS_1} a_k(\bx) - \sum_{k \in \sS_2} a_k(\bx) \ge
0$.
The two sets of inequalities show the choices of subsets of $\sN$ for
each of the two possible choices of $\sS_1$ (i.e. either $\emptyset$
or $\{1\}$). The above conditions represent a disjunction of
conjunctions.

Similarly, employing the first and third conditions of the theorem to
our running example gives us:
{\footnotesize 
  \begin{equation}
    \begin{cases}
      w_0 \ge 0,                       & \text{or $w_0 + a_1(\bx)\ge 0$}            \\
      w_0 - a_2(\bx) \ge 0,            & \text{or $w_0+ a_1(\bx) - a_2(\bx) \ge 0$} \\
      w_0 - a_3(\bx) \ge 0,            & \text{or $w_0+ a_1(\bx) - a_3(\bx) \ge 0$} \\
      w_0 - a_2(\bx) - a_3(\bx) \ge 0, & \text{or $w_0+ a_1(\bx) - a_2(\bx)$}                 \\
                                       & \text{\hspace{0.8cm}$- a_3(\bx) \ge 0$}
    \end{cases}
  \end{equation}
} 
Note that unlike the previous case, this gives us a condition that
is a conjunction of disjunctions.

The complete proof of Theorem~\ref{theorem:main} is given in the
supplementary material; here we give a sketch. To prove that
condition~\ref{theorem:main:condition1} implies
condition~\ref{theorem:main:condition2} we construct a specific subset
$\sS_1^*$ of $\sP$, where
$\mathcal{S}_1^* = \{k: k \in \mathcal{P} \text{ and } a_k(\mathbf{x})
\ge 0\}$ and this $\sS_1^*$ has the desired property. To prove
condition~\ref{theorem:main:condition2} implies
condition~\ref{theorem:main:condition1} we make use of a specific
subset $\sS_2^*$ of $\sN$, where
$\mathcal{S}_2^* = \{k: k \in \mathcal{N} \text{ and } a_k(\mathbf{x})
\ge 0\}$, to show the example $\bx$ has a positive label. That
condition~\ref{theorem:main:condition1} implies
condition~\ref{theorem:main:condition3} is a direct result of
condition~\ref{theorem:main:condition1} implying
condition~\ref{theorem:main:condition2}. Finally, to prove
condition~\ref{theorem:main:condition3} implies
condition~\ref{theorem:main:condition1} we use the same $\sS_2^*$
again to show $\bx$ has a positive label.

\paragraph{Discussion.} 
The only difference between the second and the third conditions of the
theorem is the order of the universal and existential quantifiers over
the positive and negative hidden units, $\sP$ and $\sN$ respectively.
More importantly, in both cases, the inequality condition over the
subsets $\sS_1$ and $\sS_2$ is identical. Normally, swapping the order
of the quantifiers does not give us an equivalent statement; but here,
we see that doing so retains meaning because, in both cases, the
output is positive for the corresponding input.

For any subsets $\sS_1\subseteq \sP$ and $\sS_2 \subseteq \sN$, we can
write the inequality condition as a Boolean function
$B_{\sS_1, \sS_2}$:
\begin{equation}\label{eq:boolean_var}
  B_{\sS_1, \sS_2}(\bx) = 
  \begin{cases}
    \text{true}, & w_0 + \sum_{k \in \sS_1} a_k(\bx) \\
    & - \sum_{k \in \sS_2} a_k(\bx) \ge 0 \\
    \text{false}, & w_0 + \sum_{k \in \sS_1} a_k(\bx)\\
    & - \sum_{k \in \sS_2} a_k(\bx) < 0
  \end{cases}
\end{equation}
If the sizes of the positive and negative subsets are $n_1$ and $n_2$
respectively (i.e, $n_1 = |\sP|$ and $n_2 = |\sN|$), then we know that
$\sP$ has $2^{n_1}$ subsets and $\sN$ has $2^{n_2}$ subsets. Thus,
there are $2^{n_1 + n_2}$ such Boolean functions. Then, by virtue of
conditions 1 and 2 of theorem~\ref{theorem:main}, we have\footnote{We
  write $y$ as a Boolean with $y=1$ and $y=-1$ representing true and
  false respectively.}
\begin{equation*}
  y = \lor_{\sS_1 \subseteq \sP} \pb{\land_{\sS_2 \subseteq \sN} B_{\sS_1, \sS_2}(\bx)},
\end{equation*}
where $\land_{\sS_2}$ indicates a conjunction over all different
subsets $\sS_2$ of $\sN$, and $\lor_{\sS_1}$ indicates a disjunction
over all different subsets $\sS_1$ of $\sP$. This expression is in the
disjunctive normal form (DNF), where each conjunct contains $2^{n_2}$
$B$'s and there are $2^{n_1}$ such terms. Since each $B$ simplifies
into a hyperplane in the input space, this characterizes the decision
boundary of the \relu{} network as a DNF expression over these
hyperplane decisions.

Similarly, by conditions 1 and 3, we have
%
$y = \land_{\sS_2} \bigl[\lor_{\sS_1} B_{\sS_1, \sS_2}(\bx)\bigr]$.
%
This is in the conjunctive normal form (CNF), where each disjunctive
clause contains $2^{n_1}$ Boolean values and there are $2^{n_2}$  such
clauses.

An corollary is that if the hidden units of the \relu{} network are
all positive (or negative), then the equivalent threshold network is a
pure disjunction (or conjunction).


\section{Comparing \relu{} and threshold networks}
\label{sec:transformations}

In the previous section, we saw that \relu{} networks can express
Boolean functions that correspond to much larger threshold networks.
Of course, threshold activations are not generally used in
applications; nonetheless, they are well understood theoretically {\em
  and} they emerge naturally as a result of the analysis above. Using
threshold functions as a vehicle to represent the decision boundaries
of \relu{} networks, naturally leads to two related questions that we
will address in this section. First, given an arbitrary \relu{}
network, can we construct an equivalent threshold network? Second,
given an arbitrary threshold network, how can we represent it using
\relu{} network?

\subsection{Converting from \relu{} to Threshold}
\label{subsec:relu-to-threshold}

Theorem~\ref{theorem:main} essentially gives us a constructive way to
represent an arbitrary two layer \relu{} network given in
Eq.~\eqref{eq:general_relu} as a three-layer threshold network.
For every choice of the subsets $\sS_1$ and $\sS_2$ of the positive
and negative units, we can define a Boolean function
$B_{\sS_1, \sS_2}$ as per Eq.~\eqref{eq:boolean_var}.
By definition, each of these is a threshold unit, giving us
$2^{n_1+n_2}$ threshold units in all. (Recall that $n_1$ and $n_2$ are
the sizes of $\sP$ and $\sN$ respectively.) Since the decision
function is a CNF or a DNF over these functions, it can be represented
using a two-layer network over the $B$'s, giving us three layers in
all. We put all $2^{n_1+n_2}$ threshold units in the first hidden
layer, separated into $2^{n_1}$ groups, with each group comprising of
$2^{n_2}$ units.

Figure \ref{fig:dnf} shows the threshold network corresponding to our
running example from Eq. (\ref{eq:example_relu}). For brevity, we use
the notation $B_{i,j}$ to represent the first hidden layer, with $i$
and $j$ indexing over the subsets of $\sP$ and $\sN$ respectively. The
$B$'s can be grouped into two groups, with units in each group sharing
the same subset $\sS_1$ but with different $\sS_2$. Note that, these
nodes are linear threshold units corresponding to the inequalities in
in Fig. \ref{fig:conditions-for-positive}.
In the second hidden layer, we have one threshold unit connected to
each group of units in the layer below. The weight for each connection
unit is 1 and the bias is $2^{n_2}-1$, effectively giving us a
conjunction of the previous layer nodes. The second hidden layer has
$2^{n_1}$ such units. Finally, we have one unit in the output layer,
with all weights being $1$ and bias being $1-2^{n_1}$, simulating a
disjunction of the decisions of the layer below. As discussed in the
previous section, we can also construct a threshold networks using
CNFs, with $2^{n_1+n_2}$ units in the first hidden layer, and
$2^{n_2}$ units in the second hidden layer.
\begin{figure*}[ht]
  \centering
  \begin{tikzpicture}[scale=0.95]
  
  \node[draw=black, rectangle, minimum height = 0.8cm, minimum width=13cm, anchor=south west] (input)at (0,0) {Input vector $\bx$};
  
  \foreach \i in {0, 1, 2, 3} {
    \pgfmathsetmacro\x{\i * 1.5 + 1}

    \node[draw=black,circle,minimum size=1cm](b0\i) at(\x,2.5){{\footnotesize $B_{0, \i}$}};

    \pgfmathsetmacro\x{\i * 1.5 + 7.5}

    \node[draw=black,circle,minimum size=1cm](b1\i) at(\x,2.5){{\footnotesize $B_{1, \i}$}};
   }

  \node[draw=black, circle, left of=b00, yshift=1cm] (b0) {{\footnotesize 1}};
  \node[draw=black, circle, right of=b13, yshift=1cm] (b1)  {{\footnotesize 1}};

  \node[draw=black, circle, minimum size=1cm] (c0) at ([yshift=2cm]$(b00)!0.5!(b03)$) {{\footnotesize $C_0$}};
  \node[draw=black, circle, minimum size=1cm] (c1) at ([yshift=2cm]$(b10)!0.5!(b13)$) {{\footnotesize $C_1$}};
  \node[draw=black, circle, right of=c1, yshift=1cm] (b) {{\footnotesize 1}};

  \foreach \i in {0,1, 2, 3} {
    \draw[->, thick] (b0\i) -- (c0); 
    \draw[->, thick] (b1\i) -- (c1); 
  }
  \draw[->, thick] (b0) -- (c0); 
  \draw[->, thick] (b1) -- (c1); 

  \node[draw=black, circle, minimum size=1cm] (y) at ([yshift=2cm]$(c0)!0.5!(c1)$){$y$};

  \draw[->, thick] (c0) -- (y); 
  \draw[->, thick] (c1) -- (y); 
  \draw[->, thick] (b) -- (y);  

  \draw[gray] ([xshift=-0.3cm, yshift=-0.2cm]b00.south west) rectangle ([xshift=0.3cm, yshift=0.2cm]b03.north east);
  \draw[gray] ([xshift=-0.3cm, yshift=-0.2cm]b10.south west) rectangle ([xshift=0.3cm, yshift=0.2cm]b13.north east);

  \draw[->, very thick] (input.north) -- ([yshift=-0.6cm]$(b01)!0.5!(b02)$); 
  \draw[->, very thick] (input.north) -- ([yshift=-0.6cm]$(b11)!0.5!(b12)$); 

  \node[draw=gray, dashed] at (6.5, 1.5) {Fully connected};

  \node[draw=gray, dashed, anchor=north west] at ([yshift=-0.2cm,
  xshift=-0.3cm]b00.south west) {{\footnotesize $\sS_1 = \emptyset$
      (See Fig. \ref{fig:conditions-for-positive}, left)}};

  \node[draw=gray, dashed, anchor=north east] at ([xshift=0.3cm, yshift=-0.2cm]b13.south east) {{\footnotesize $\sS_1 = \{1\}$ (See Fig. \ref{fig:conditions-for-positive}, right)}};

  \node[draw=gray, dashed] (conjLabel)at ([yshift=-0.5cm]$(c0)!0.5!(c1)$) {Conjunction units};
  \draw[dashed, ->] (conjLabel.west) -- (c0);
  \draw[dashed, ->] (conjLabel.east) -- (c1);

  \node[draw=gray, dashed, left of=y, xshift=-2cm] (disjLabel) {Disjunction unit};
  \draw[dashed, ->] (disjLabel.east) -- (y);

\end{tikzpicture}

  \caption{A threshold network corresponding to the running example.
    The dotted boxes label the various components of the network. See
    the text above for details. }
  \label{fig:dnf}
\end{figure*}
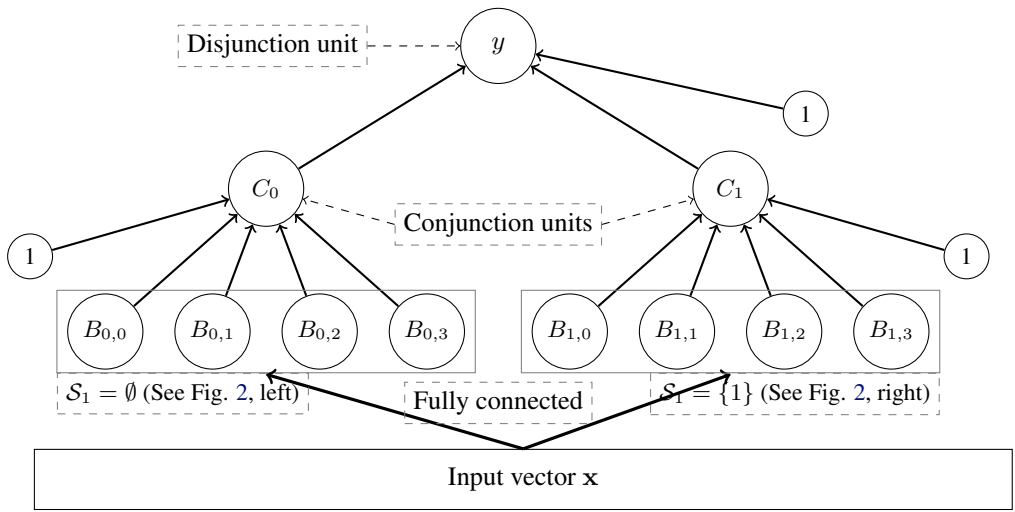

\subsection{Boolean Expressiveness of \relu{} Networks}
\label{sec:boolean-expressiveness}

Our main theorem shows that for an arbitrary two-layer ReLU network,
we \emph{can always} represent it using a three-layer threshold
network in which the number of threshold units is exponentially more
than the number of ReLUs. However, this does not imply that actually
need that many thresholds units. A natural question to ask is: can we
represent the same ReLU network without the expense of exponential
number of threshold units? In general, the answer is no, because there
are families of ReLU network that need at least an exponential number
of threshold units to represent.

We can formalize this for the case of \relu{} networks where the
hidden nodes are either all positive or negative -- that is, either
$\sP$ or $\sN$ is the empty set. We restrict ourselves to this set
because we can represent such networks using a two-layer threshold
network rather than the three-layer one in the previous section.
We will consider the set of all Boolean functions in $d$ dimensions
expressed by such a \relu{} network with $n$ hidden units. Let
$\Gamma_R(n, d)$ represent this set. Similarly, let $\Gamma_T(n, d)$
denote the set of such functions expressed by threshold networks with
$n$ hidden units. We can summarize the Boolean expressiveness of
\relu{} networks via the following two-part theorem:

\begin{theorem}\label{thm:boolean-expressiveness}
  For rectifier networks with $n > 1$ hidden units, all of which are
  either positive or negative, and for all input dimensionalities $d\ge n$,
  we have
  \begin{enumerate}
  \item $\Gamma_R(n, d) \subseteq \Gamma_T(2^n-1, d)$, and,
  \item $\Gamma_R(n, d) \not\subset \Gamma_T(2^n-2, d)$.
  \end{enumerate}
\end{theorem}

Before seeing the proof sketch, let us look at an intuitive
explanation of this result. This theorem tells us that the upper bound
on the number of threshold units corresponding to the \relu{} network
is a tight one. In other words, not only can \relu{} networks express
Boolean functions that correspond to much larger threshold networks,
there are some \relu{} networks that can {\em only} be expressed by
such large networks! This theorem may give us some intuition into the
successes of \relu{} networks.

Note, however, that this theorem does not resolve the question of what
fraction of all ReLU networks can be expressed using fewer than
exponential number of hidden threshold units. Indeed, if the inputs
were only Boolean and the output was allowed to be real-valued, then
\citet[][theorem 6]{Martens2013} show that a two layer ReLU network
can be simulated using only a quadratic number of threshold units. (In
contrast, Theorem~\ref{thm:boolean-expressiveness} above considers the
case of real-valued inputs, but Boolean outputs.)

The proof is based on Theorem~\ref{theorem:main}. From
Theorem~\ref{theorem:main}, we know the decision boundary of the
rectifier network with $n$ positive ReLUs (with $\sN$ being the empty
set) can be determined from $2^n-1$ hyperplanes. In other words, the
region of the input space that is classified as negative is the
polytope that is defined by the intersection of these $2^n-1$
half-planes. We can show by construction that there are rectifier
networks for which the negative region is a polytope defined by
$2^n -1$ bounding surfaces, thus necessitating {\em each} of the
$2^n-1$ half-planes as predicted by the main theorem, irrespective of
how the corresponding threshold network is constructed. In other
words, the number of threshold units cannot be reduced. The proof for
the case of all negative ReLUs is similar. The complete proof of the
theorem is given in the supplementary material.

\subsection{Converting Thresholds to \relus}
\label{subsec:threshold-to-relu}

So far, we have looked at threshold networks corresponding to \relu{}
networks. The next question we want to answer is under what condition
we can use \relus{} to represent the same decision boundary as a
threshold network. In this section, we show a series of results that
address various facets of this question. The longer proofs are in the
appendices.

First, with no restrictions in the number of \relus{} in the hidden
layer, then we can always construct a rectifier network that is
equivalent to a threshold network. In fact, we have the following
lemma:
\begin{lemma}\label{lemma:double_relu}
  Any  threshold network with $n$ units can be approximated to
  arbitrary accuracy by a rectifier network with $2n$ units.
\end{lemma}
\begin{proof}
Consider a threshold unit with weight vector $\mathbf{v}$ and bias $d$, we have
\begin{equation*}
  \text{sgn}(\mathbf{v} \cdot \mathbf{x} + d) \simeq \frac{1}{\epsilon} \bigl[R(\mathbf{v} \cdot \mathbf{x} + d + \epsilon) - R(\mathbf{v} \cdot \mathbf{x} + d - \epsilon)\bigr] - 1.
\end{equation*}
where $\epsilon$ is an arbitrary small number which determines the
approximation accuracy.
\end{proof}
This result is akin to the simulation results from
\citep{maass1994comparison} that compares threshold and sigmoid networks.

Given the exponential increase in the size of the threshold network to
represent a \relu{} network (Theorem
\ref{thm:boolean-expressiveness}), a natural question is whether we
can use only logarithmic number of \relus{} to represent any arbitrary
threshold network. In the general case, the following lemma points out
that this is not possible.
\begin{lemma}\label{lemma:counter_example}
  There exists a two-layer network with $n$ hidden threshold units for
  which it is not possible to construct an equivalent two-layer
  \relu{} network with fewer number of hidden units.
\end{lemma}

We provide such an example with $n$ hidden threshold units in the
supplementary material. This lemma, in conjunction with Theorem
\ref{thm:boolean-expressiveness} effectively points out that by
employing a rectifier network, we are exploring a {\em subset} of much
larger threshold networks.

Furthermore, despite the negative result of the lemma, in the general
case, we can identify certain specific threshold networks that can be
compressed into logarithmically smaller ones using \relus. Suppose we
wish to compress a two layer threshold network with three sign hidden
units into a \relu{} network with $\ceil{\log_2{3}} =2$ hidden units.
The sign network can be represented by
\begin{eqnarray*}
  y = \text{sgn}(2 + \threshold{\bv_1 \cdot \bx + d_1} + \threshold{\bv_2 \cdot \bx + d_2} \\
  + \threshold{\bv_3 \cdot \bx + d_3})
\end{eqnarray*}
Suppose one of the weight vectors can be written as the linear
combination of the other two but its bias can not. That is, for some
$p$ and $q$, if $\bv_3 = p \bv_1 + q \bv_2$ and $d_3 \ne pd_1 + qd_2$.
Then, we can construct the following equivalent \relu{} network that is
equivalent:
\begin{align*}
   y = &\threshold{-1 + \rectifier{\bu_1 \cdot \bx + b_1} +\rectifier{\bu_2 \cdot \bx + b_2}}, \\
   \mbox{where  \quad} r &= \frac{1}{d_3 - p d_1 - q d_2},\\
  \bu_1 &= pr\bv_1,\\
  \bu_2 &= qr\bv_2, \\
  b_1 &= prd_1+1, \\
 b_2 &= qrd_2 + 1.
\end{align*}
This equivalence can be proved by applying Theorem \ref{theorem:main}
to the constructed \relu{} network. It shows that in two dimensions,
we can use two \relus{} to represent three linearly independent sign
units.

We can generalize this result to the case of a two-layer threshold
network with $2^n$ hidden threshold units that represents a
disjunction over the hidden units. The goal is to find that under what
condition we can use only $n$ rectifier units to represent the same
decision. To do so, we will use binary encoding matrix $T_n$ of size
$n\times 2^n$ whose $i^{th}$ column is the binary representation of
$i-1$. For example, the binary encoding matrix for $n=3$ is given by
$T_3$,
\begin{equation*}
  T_3 = 
  \begin{bmatrix}
    0 & 0 & 0 & 0 & 1 & 1 & 1 & 1 \\
    0 & 0 & 1 & 1 & 0 & 0 & 1 & 1 \\
    0 & 1 & 0 & 1 & 0 & 1 & 0 & 1 \\
  \end{bmatrix}
\end{equation*}
\begin{lemma}\label{lemma:encoding}
  Consider a two-layer threshold network with $2^n$ threshold units in
  the hidden layer whose output represents a disjunction over the
  hidden units, i.e., the final output is positive if and only if at
  least one of the hidden-unit outputs is positive. That is, 
  \begin{equation}\label{eq:disjunction}
    y = \threshold{2^n-1 + \sum_{k=1}^{2^n}  \threshold{\bv_k \cdot \bx + d_k}}.
  \end{equation}
  This decision can be represented using a two-layer rectifier
  network with $n$ hidden units, if the weight parameters of the
  threshold units can be factored in the following form:
  \begin{equation}\label{eq:combined}
    \begin{bmatrix}
      \bv_1 & \cdots & \bv_{2^n} \\
      d_1 & \cdots & d_{2^n}
    \end{bmatrix}
    =
    \begin{bmatrix}
      \mathbf{u}_1 & \cdots & \mathbf{u}_n & \mathbf{0}\\
      b_1 & \cdots & b_n & w_0
    \end{bmatrix}
    \begin{bmatrix}
      T_n \\
      \be_{2^n}
    \end{bmatrix}
  \end{equation}
  where $\be_{2^n}$ is a $2^n$ dimensional row vector of all ones and
  ${\bf 0}$ is a vector of all zeros.
\end{lemma}
\begin{proof}
If the weight parameters $\mathbf{v}_k$ and $d_k$ can be written in the form as in Eq.~\eqref{eq:combined}, then we can construct the two-layer rectifier network,
\begin{equation}\label{eq:relu_equivalent}
  y = \text{sgn} \bigl[w_0 + \sum_{k=1}^n R(\mathbf{u}_k \cdot \mathbf{x} + b_k)\bigr].
\end{equation}
Then by virtue of theorem~\ref{theorem:main}, the decision boundary of
the rectifier network in Eq.~\eqref{eq:relu_equivalent} is the same as
the decision boundary of the threshold network in
Eq.~\eqref{eq:disjunction}.
\end{proof}
Note that this lemma only identifies sufficient conditions for the
logarithmic reduction in network size. Identifying both necessary and
sufficient conditions for such a reduction is an open question.


\section{Hidden Layer Equivalence}
\label{sec:multiclass}
Lemma \ref{lemma:encoding} studies a specific threshold network, where
the output layer is a disjunction over the hidden layer units. For
this network, we can define an different notion of equivalence between
networks by studying the hidden layer activations. We do so by
interpreting the hidden layer state of the network as a specific kind
of a multiclass classifier that either rejects inputs or labels them.
If the output is negative, then clearly none of the hidden layer units
are active and the input is rejected. If the output layer is positive,
then at least one of the hidden layer units is active and the
multiclass label is given by the maximum scoring hidden unit, namely
$\argmax_k \bv_k \cdot \bx + d_k$.

For threshold networks, the number of hidden units is equal to the
number of classes. The goal is to learn the same concept with
rectifier units, hopefully with fewer rectifier units than the number
of classes. Suppose a \relu{} network has $n$ hidden units, then its
hidden layer prediction is the highest scoring hidden unit of the
corresponding threshold network that has $2^n$ hidden units. We now
define {\em hidden layer equivalence} of two networks as follows: A
threshold network and a \relu{} network are equivalent if both their
hidden layer predictions are identical.

We already know from lemma~\ref{lemma:encoding} that if the weight
parameters of the true concept satisfy Eq.~\eqref{eq:combined}, then
instead of learning $2^n$ threshold units we can just learn $n$
rectifier units. For simplicity, we write Eq.~\eqref{eq:combined} as
$V = UT$ where
\begin{eqnarray*}
  V = \begin{bmatrix}
    \mathbf{v}_1 & \cdots & \mathbf{v}_{2^n} \\
    d_1 & \cdots & d_{2^n}
  \end{bmatrix} &
  U = 
  \begin{bmatrix}
    \mathbf{u}_1 & \cdots & \mathbf{u}_n & \mathbf{0}\\
    b_1 & \cdots & b_n & w_0
  \end{bmatrix} 
\end{eqnarray*}
and
\begin{equation*}
  T = 
  \begin{bmatrix}
    T_n \\
    \be_{2^n}
  \end{bmatrix}
\end{equation*}
For simplicity of notation, we will assume that the input features
$\bx$ includes a constant bias feature in the last position. Thus,
the vector $V^T\bx$ represents the pre-activation score for each class.

Now, we consider threshold networks with parameters such that there is
{\em no} \relu{} (defined by the matrix $U$) that satisfies this
condition. Instead, we find a rectifier network with parameters $U$
that satisfies the following condition:
\begin{equation}\label{eq:norm}
  U = \text{argmin}_U \Vert (V-UT)^T\Vert_{\infty},
\end{equation}
Here $\Vert\cdot \Vert_\infty$ is the induced infinity norm, defined
for any matrix $A$ as
$\Vert A \Vert_\infty = \sup_{x\ne 0} \frac{\Vert Ax\Vert_\infty}{\Vert x\Vert_\infty}$.

If we have a matrix $U$ such that $V$ and $UT$ are close in the sense
of induced infinity norm, then we have the following about their
equivalence.
\begin{theorem}\label{theorem:bound}
  If the true concept of a $2^n$-class classifier is given by a
  two-level threshold network in Eq.~\eqref{eq:disjunction}, then we
  can learn a two-layer rectifier network with only $n$ hidden units
  of the form in Eq.~\eqref{eq:relu_equivalent} that is hidden layer
  equivalent to it, if for any example $\bx$, we have
  \begin{equation}
    \Vert (V-UT)^T\Vert_{\infty} \le \frac{\gamma(\bx)}{2 \Vert \bx \Vert_\infty},
  \end{equation}
  where $\gamma(\bx)$ is the multiclass margin for $\bx$, defined as
  the difference between its highest score and second-highest scoring
  classes.
\end{theorem}
The proof, in the supplementary material, is based on the intuition
that for hidden layer equivalence, as defined above, {\em only}
requires that the highest scoring label needs to be the same in the
two networks rather than the actual values of the scores. If $V$ and
$UT$ are closed in the sense of induced infinity norm, then the
highest scoring hidden unit will be invariant regardless of which
network is used.


\section{Experiments}
\label{sec:experiments}
We have seen that every two-layer rectifier network expresses the
decision boundary of a three-layer threshold network. If the output
weights of the former are all positive, then a two-layer threshold
network is sufficient. (See the discussion in \S
\ref{sec:boolean-expressiveness}.) However, the fact that rectifier
network can express the same decision boundary more compactly does not
guarantee learnability because of optimization issues. Specifically,
in this section, we study the following question using synthetic data:
{\em Given a rectifier network and a threshold network with same
  decision boundary, can we learn one using the data generated from
  another using backpropagation?}

\subsection{Data generation}
\label{sec:data-gen}
We use randomly constructed two-layer rectifier networks to generate
labeled examples. To do so, we specify various values of the input
dimensionality and the number of hidden \relu{} units in the network.
Once we have the network, we randomly generate the input points and
label them using the network.
Using generated data we try to recover both the rectifier network and
the threshold network, with varying number of hidden units. 
We considered input dimensionalities 3, 10 and 50 and in each case,
used 3 or 10 hidden units. This gave us six networks in all. For each
network, we generated 10000 examples and 1500 of which are used as test
examples.

\subsection{Results and Analysis}
\label{sec:results}
For each dataset, we compare three different network architectures.
The key parameter that varies across datasets is $n$, the number of hidden
\relu{} units in the network that generated the data.
The first setting learns using a \relu{} network with $n$ hidden
units. The second setting uses the activation function $\tanh(cx)$,
which we call the compressed $\tanh$ activation. For large values of
$c$, this effectively simulates the threshold function. In the second
setting, the number of hidden units is still $n$. The final setting
learns using the compressed $\tanh$, but with $2^n$ hidden units
following \S \ref{subsec:relu-decisions}.

\begin{figure}[htbp] 
   \centering
    \includegraphics[width=0.5\textwidth]{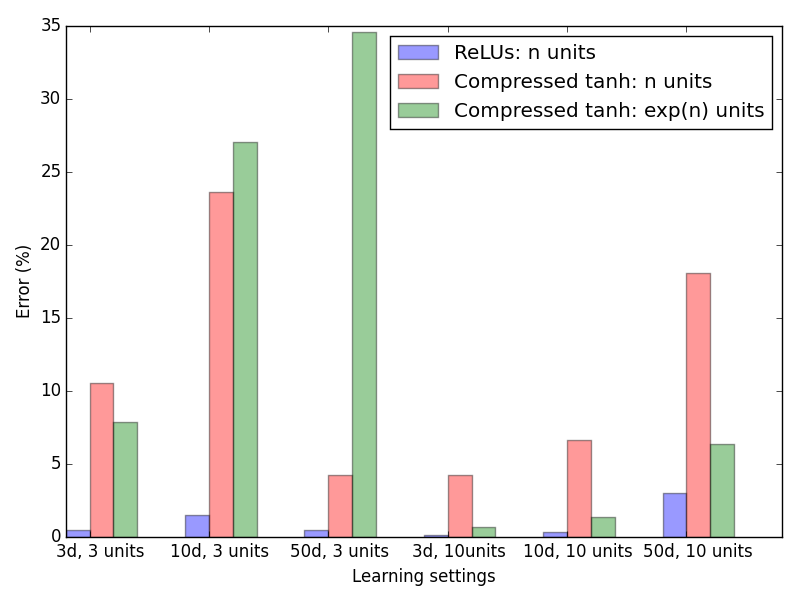} 
    \caption{{\footnotesize Test error for different learning
        settings. The x-axis specifies the number of ReLUs used for
        data generalization and the dimensionality. Each dataset is
        learned using ReLU and compressed $\tanh$ activations with
        different number of hidden units. Learning rate is selected
        with cross-validation from
        $\{10^0, 10^{-1}, 10^{-2}, 10^{-3}, 10^{-4}\}$.
        $L_2$-regularization coefficient is $10^{-4}$. We use
        early-stopping optimization with a maximum of 1000 epochs. The
        minibatch size is 20. For the compressed $\tanh$, we set
        $c=10000$.}}
   \label{fig:exp}
\end{figure}

Figure~\ref{fig:exp} shows the results on the six datasets. These
results verify several aspects of our theory. First, learning using
ReLUs always succeeds with low error (purple bars, left). This is
expected -- we know that our hypothesis class can express the true
concept and training using backpropagation can successfully find it.
Second, learning using compressed $\tanh$ with same number of units
cannot recover the true concept (red bars, middle). This performance
drop is as expected, since compressed $\tanh$ is just like sign
activation, and we know in this case we need exponentially more hidden
units.

Finally, the performances of learning using exponential number of
compressed $\tanh$ (green bars, right) are not always
good.\footnote{We also evaluated the performance on the training set.
  The training errors for all six datasets for all three learning
  scenarios are very close to test error, suggesting that over-fitting
  is not an issue.} In this case, from the analysis in \S
\ref{sec:expressiveness}, we know the hypothesis can certainly express
the true concept; yet learning does not always succeed! In fact, for
the first three groups, where we have three ReLUs for data generation,
the error for the learned classifier is rather large, suggesting that
even though the true concept can be expressed, it is not found by
backpropagation. For the last three groups, where we have 10 hidden
ReLUs for data generation, using exponential number of compressed
$\tanh$ does achieve better performance. We posit that this
incongruence is due to the interplay between the non-convexity of the
objective function and the fact that the set of functions expressed by
threshold functions is larger (a consequence of lemma
\ref{lemma:counter_example}).


\section{Conclusions}
\label{sec:conclusions}

In this paper, we have presented a novel analysis of the
expressiveness of rectifier neural networks. Specifically, for binary
classification we showed that even though the decision boundary of
two-layer rectifier network can be represented using threshold unit
network, the number of threshold units required is exponential.
Further, while a corresponding general logarithmic reduction of
threshold units is not possible, for specific networks, we
characterized sufficient conditions for reducing a threshold network
to a much smaller rectifier network. We also presented a relaxed
condition where we can approximately recover a rectifier network that
is hidden layer equivalent to an exponentially larger threshold
network.

Our work presents a natural next step: can we use the equivalence of
the expressiveness results given in this paper to help us study the
sample complexity of rectifier networks? Another open question is the
generalization of these results to deep networks. Finally, from our
experiments we see that expressiveness is not enough to guarantee
learnability. Studying the interplay of expressiveness, sample
complexity and the convexity properties of the training objective
function for rectifier networks represents an exciting direction of
future research.


\section*{Acknowledgements}
\label{sec:acknowledgements}

We are grateful to the reviewers for their invaluable comments and
pointers to related work. Xingyuan was partly supported by the School
of Computing PhD fellowship.


\bibliography{relu-icml2016}
\bibliographystyle{icml2016}

\newpage

\icmltitlerunning{Expressiveness of Rectifier Networks (Supplementary Material)}

\twocolumn[
\icmltitle{Expressiveness of Rectifier Networks \\ (Supplementary Material)}


\icmlkeywords{machine learning, artificial neural network, rectified linear unit, expressiveness}

\vskip 0.3in
]

\section*{Proof of Theorem~\ref{theorem:main}}\label{appendix:main_theorem_appendix}
In this section, we prove our main theorem, Theorem~\ref{theorem:main},
which is repeated here for convenience.
\newtheorem*{theorem-main-repeated}{Theorem \ref{theorem:main}}
\begin{theorem-main-repeated} Consider a two-layer rectifier network
  with $n$ hidden units represented in its general form
  (Eq.~\eqref{eq:general_relu}). Then, for any input $\bx$, the
  following conditions are equivalent:
  \begin{enumerate}
  \item The network classifies the example $\bx$ as positive.
  \item There exists a subset $\sS_1$ of $\sP$ such that, for every
    subset $\sS_2$ of $\sN$, we have
    $w_0 + \sum_{k \in \sS_1} a_k(\bx) - \sum_{k \in \sS_2} a_k(\bx)
    \ge 0$.
  \item For every subset $\sS_2$ of $\sN$, there exists a subset
    $\sS_1$ of $\sP$ such that
    $w_0 + \sum_{k \in \sS_1} a_k(\bx) - \sum_{k \in \sS_2} a_k(\bx)
    \ge 0$.
  \end{enumerate}
\end{theorem-main-repeated}

{\bf Step 1: Equivalence of conditions 1 and 2}

Let us prove that condition 1 implies condition 2 first. We can
construct a subset $\mathcal{S}_1^*$ of $\mathcal{P}$
\begin{equation*}
  \mathcal{S}_1^* = \{k: k \in \mathcal{P} \text{ and } a_k(\mathbf{x}) \ge 0\},
\end{equation*}
such that 
\begin{equation*}
  \sum_{k \in \mathcal{P}} R(a_k(\mathbf{x})) = \sum_{k \in \mathcal{S}_1^*} a_k(\mathbf{x}).
\end{equation*} 
The example $\bx$ is classified as positive implies that
\begin{equation*}
  w_0 + \sum_{k \in \mathcal{S}_1^*} a_k(\mathbf{x})  \ge \sum_{k \in \mathcal{N}} R(a_k(\mathbf{x})).
\end{equation*}
For any subset $\mathcal{S}_2$ of $\mathcal{N}$, we have 
\begin{equation*}
  \sum_{k \in \mathcal{N}} R(a_k(\mathbf{x})) \ge \sum_{k \in \mathcal{S}_2} R(a_k(\mathbf{x})) \ge \sum_{k \in \mathcal{S}_2} a_k(\mathbf{x}).
\end{equation*}
Therefore, for any subset $\mathcal{S}_2$ of $\mathcal{N}$,
\begin{equation*}
  w_0 + \sum_{k \in \mathcal{S}_1^*} a_k(\mathbf{x}) \ge \sum_{k \in \mathcal{S}_2} a_k(\mathbf{x}).
\end{equation*}

Now, we need to show that condition 2 implies condition 1. Assume
there is a subset $\mathcal{S}_1$ of $\mathcal{P}$ such that for any
subset $\mathcal{S}_2$ of $\mathcal{N}$,
$w_0 + \sum_{k \in \mathcal{S}_1} a_k(\mathbf{x}) \ge \sum_{k \in
  \mathcal{S}_2} a_k(\mathbf{x}) $.
Let us define a specific subset $\mathcal{S}_2^*$ of $\mathcal{N}$,
\begin{equation*}
  \mathcal{S}_2^* = \{k: k \in \mathcal{N} \text{ and } a_k(\mathbf{x}) \ge 0\},
\end{equation*}
such that
\begin{equation*}
\sum_{k \in \mathcal{N}} R(a_k(\mathbf{x})) = \sum_{k \in \mathcal{S}_2^*} a_k(\mathbf{x}).
\end{equation*} 
We know that
\begin{equation*}
  \sum_{k \in \mathcal{P}} R(a_k(\mathbf{x})) \ge  \sum_{k \in \mathcal{S}_1} R(a_k(\mathbf{x})) \ge \sum_{k \in \mathcal{S}_1} a_k(\mathbf{x})
\end{equation*}
and
\begin{equation*}
  w_0 + \sum_{k \in \sS_1} a_k(\bx) \ge \sum_{k \in \sS_2^*} a_k(\bx).
\end{equation*}
Therefore,
\begin{equation*}
w_0 + \sum_{k \in \mathcal{P}} R(a_k(\mathbf{x})) \ge \sum_{k \in \mathcal{N}} R(a_k(\mathbf{x}))
\end{equation*}
which means that the decision function $y$ in
Eq.~\eqref{eq:general_relu} is positive.

{\bf Step 2: Equivalence of conditions 1 and 3}

That condition 1 implies condition 3 holds by virtue of the first
part of the previous step. We only need to prove that condition 3
implies 1 here. Assume for all subset $\mathcal{S}_2$ of $\mathcal{N}$
there is a subset $\mathcal{S}_1$ of $\mathcal{P}$ such that $w_0 +
\sum_{k \in \mathcal{S}_1} a_k(\mathbf{x}) - \sum_{k \in
  \mathcal{S}_2} a_k(\mathbf{x}) \ge 0$. Use the same $\sS_2^*$
defined in previous step
\begin{eqnarray*}
  w_0 + \sum_{k \in \mathcal{P}} R(a_k(\mathbf{x})) 
  &\ge& w_0 + \sum_{k \in \mathcal{S}_1} R(a_k(\mathbf{x}))  \\
  &\ge& w_0 + \sum_{k \in \mathcal{S}_1} a_k(\mathbf{x}) \\
  &\ge& \sum_{k \in \mathcal{S}_2^*} a_k(\mathbf{x}) \\
  &=& \sum_{k \in \mathcal{N}} R(a_k(\mathbf{x}))
\end{eqnarray*}
Therefore, the decision function $y$ in Eq.~\eqref{eq:general_relu}
is positive. \qed


\section*{Proof of Theorem~\ref{thm:boolean-expressiveness}}
\label{appendix:tight_bound}

The first part of this theorem says that every rectifier network with
$n$ hidden units that are all positive or negative can be represented
by a threshold network with $2^n-1$ hidden units. This is a direct
consequence of the main theorem.

The second part of the theorem says that for any $n$, there are
families of rectifier networks whose equivalent threshold network will
need an exponential number of hidden threshold units. We prove this
assertion constructively by providing one such rectifier network.
Consider the decision function of a two-layer rectifier network
\begin{equation*}
y = \text{sgn} [-1 + R(x_1) + R(x_2) + \dots + R(x_n)]
\end{equation*}
where $x_i$ is the $i^{th}$ component of the input vector (recall that
the dimensionality of the input $d \ge n$). From
Theorem~\ref{theorem:main} the decision boundary of this network can
be determined by $2^n-1$ hyperplanes of the form
$-1+\sum_{i \in \sS} x_i = 0$, each of which corresponds a
\emph{non-empty} subset $\sS \in [n]$. The output $y$ is positive if
any of these hyperplanes classify the example as positive.

To prove that we need at least $2^n-1$ threshold units to represent
the same decision boundary, it suffices to prove that each of the
$2^n-1$ hyperplanes is needed to define the decision boundary.

Consider any hyperplane defined by \emph{non-empty} subset
$\sS \in [n]$ whose cardinality $s = |\sS|$. Let $\bar{\sS}$ be the
complement of $\sS$. Consider an input vector $\bx$ that satisfies the
following condition if $s>1$:
\begin{equation*}
\begin{cases}
1/s < x_i < 1/({s-1}), & \text{if $i \in \sS$} \\
x_i < -1/(s-1), & \text{if $i \in \bar{\sS}$}
\end{cases}
\end{equation*}
For those subsets $\sS = \{x_j\}$ whose cardinality is one, we can let
$x_j>1$ and all the other components $x_i<-x_j$. Clearly, the
rectifier network will classify the example as positive. Furthermore,
by construction, it is clear that $-1 + \sum_{i \in \sS} x_i > 0$, and
for all other subset $\sS' \in [n]$, $-1 + \sum_{i \in \sS'} x_i < 0$.
In other words, {\em only} the selected hyperplane will classify the
input as a positive one. That is, this hyperplane is required to
define the decision boundary because without it, the examples in the
above construction will be incorrectly classified by the threshold
network.

%
Therefore we we have identified the decision boundary of the given
rectifier network as an polytope with \emph{exactly} $2^n-1$ faces, by
constructing $2^n-1$ hyperplanes using $2^n-1$ \emph{non-empty}
subsets $\sS \in [n]$. To complete the proof, we need to show that
other construction methods cannot do better than our construction,
i.e., achieving same decision boundary with less hidden threshold
units. To see this, note that for each face of the decision polytope,
one needs a hidden threshold unit to represent it. Therefore no matter
how we construct the threshold network, we need at least $2^n-1$
hidden threshold units. \qed

\section*{Proof of Lemma~\ref{lemma:counter_example}}
\label{appendix:counter_example}

In this section, we provide a simple example of a two-layer threshold
network, whose decision boundary cannot be represented by a two-layer
rectifier network with fewer hidden units. Consider a threshold
network
\begin{equation*}\label{eq:three_signs}
  y = \text{sgn} [n-1 + \text{sgn}(x_1) + \text{sgn}(x_2) + \dots + \text{sgn}(x_n)],
\end{equation*}
where $x_i$ is the $i^{th}$ component of the input vector (here we
assume the dimensionality of the input $d \ge n > 1$). It is easy to
see that the decision function of this network is positive, if, and
only if at least one of the component $x_i$ is non-negative. From a
geometric point of view, each $x_i$ defines a hyperplane $x_i=0$. Let
$H_1$ be the set of $n$ such hyperplanes,
\begin{equation*}
H_1 = \{x_i=0: i \in [n]\}
\end{equation*}
These $n$ hyperplanes form $2^n$ hyperoctants and only one hyperoctant
gets negative label.

Now, suppose we construct a two-layer network with $m$ rectifier units
that has the same decision boundary as the above threshold network,
and it has the form
\begin{equation*}\label{eq:m_relus}
  y = \threshold{w_0 + \sum_{k \in \sP} \rectifier{a_k(\bx)} - \sum_{k \in \sN} \rectifier{a_k(\bx)}}
\end{equation*}
using the same notations as in Theorem~\ref{theorem:main}. From the
theorem, we know the decision boundary of the above equation is
determined by $2^m$ hyperplane equations and these $2^m$ hyperplanes
form a set $H_2$,
\begin{equation*}
H_2 = \Big\{w_0 + \sum_{k \in \sS_1} a_k(\bx) - \sum_{k \in \sS_2} a_k(\bx)= 0: \sS_1 \subseteq\sP, \sS_2\subseteq\sN\Big\}
\end{equation*}
Because we assume the rectifier network has the same decision function as the threshold network, we have
\begin{equation*}
H_1 \subseteq H_2.
\end{equation*}
Note that the hyperplanes in $H_1$ have normal vectors independent of
each other, which means there are at least $n$ hyperplanes in $H_2$
such that their normal vectors are independent. Recall that
$a_k(\bx) = \bu_k \cdot \bx + b_k$, so all normal vectors for
hyperplanes in $H_2$ can be expressed using linear combinations of $m$
vectors $\bu_1, \dots, \bu_m$. Since these $m$ vectors together define
the $n$ orthogonal hyperplanes in $H_1$, it is impossible to have
$m < n$. In other words, for this threshold network, the number of
hidden units cannot be reduced by any conversion to a \relu{} network.

\qed

\section*{Proof of Theorem~\ref{theorem:bound}}\label{appendix:bound}
In this section we prove our theorem about hidden layer equivalence,
Theorem~\ref{theorem:bound}, which is repeated here for convenience.
\newtheorem*{theorem-bound-repeated}{Theorem \ref{theorem:bound}}
\begin{theorem-bound-repeated}
  If the true concept of a $2^n$-class classifier is given by a
  two-level threshold network in Eq.~\eqref{eq:disjunction}, then we
  can learn a two-layer rectifier network with only $n$ hidden units
  of the form in Eq.~\eqref{eq:relu_equivalent} that is hidden layer
  equivalent to it, if for any example $\bx$, we have
  \begin{equation*}
    \Vert (V-UT)^T\Vert_{\infty} \le \frac{\gamma(\bx)}{2 \Vert \bx \Vert_\infty},
  \end{equation*}
  where $\gamma(\bx)$ is the multiclass margin for $\bx$, defined as
  the difference between its highest score and second-highest scoring
  classes.
\end{theorem-bound-repeated}
Let us define $\epsilon = \Vert (V-UT)^T\Vert_{\infty} \le
\frac{\gamma(x)}{2 \Vert x \Vert_\infty}$. From the definition of the
$L_\infty$ vector norm, we have
\begin{equation*}
  \Vert (V-UT)^T x\Vert_{\infty} \ge \vert((V-UT)^Tx)_k \vert
\end{equation*}
for all $x$ and all $k$. The subscript $k$ labels the $k^{th}$
component of the vector. From the definition of the induced norm we
have
\begin{equation*}
  \Vert (V-UT)^T x\Vert_{\infty} \le \epsilon \Vert x \Vert.
\end{equation*}
Combining the above two inequalities we have
\begin{equation*}
  \vert ((UT)^T x)_k - (V^Tx)_k \vert \le \epsilon \Vert x \Vert
\end{equation*}
for all $x$ and all $k$. Assuming $k^*$ is the highest scoring unit,
for $k^*$ we have
\begin{equation*}\label{eq:p1}
  (V^Tx)_{k^*} - ((UT)^Tx)_{k^*} \le \epsilon \Vert x \Vert,
\end{equation*}
For any other $k' \ne k^*$, we have
\begin{equation*}\label{eq:p2}
  ((UT)^Tx)_{k'} - (V^Tx)_{k'} \le \epsilon \Vert x \Vert.
\end{equation*}
From the definition of the margin $\gamma(x)$, we also know that
\begin{equation*}\label{eq:p3}
  (V^Tx)_{k^*} - (V^Tx)_{k'} \ge \gamma(x) \ge 2\epsilon \Vert x \Vert.
\end{equation*}
Combining the above three inequalities, we have
\begin{equation*}
  ((UT)^Tx)_{k'} \le ((UT)^Tx)_{k^*}
\end{equation*}
which means if $k^*$ is the correct class with the highest score
according to the weight parameters $V$, it will still be the highest
scoring class according to the weight parameters $UT$, even if $V
\ne UT$. \qed


\end{document}